\documentclass[letterpaper, 10 pt,conference]{ieeeconf}

\IEEEoverridecommandlockouts
\usepackage{graphicx}
\usepackage{cite}
\graphicspath{{Figures/}}
\usepackage{float,amsmath,amsfonts,stmaryrd}
\usepackage{mathtools, cuted}
\usepackage[ruled,vlined,linesnumbered]{algorithm2e}
\usepackage{algpseudocode}
\usepackage{booktabs}
\usepackage{paralist}

\newtheorem{problem}{Problem}[section]
\newtheorem{theorem}{Theorem}[section]
\newtheorem{lemma}{Lemma}[section]

\newtheorem{remark}{Remark}[section]
\newtheorem{assumption}{Assumption}[section]
\newtheorem{definition}{Definition}[section]

\usepackage{xcolor}
\definecolor{wheat}{rgb}{0.96,0.87,0.70}

\DeclareMathOperator*{\argmax}{arg\,max}

\title{\LARGE \bf
Robustness-Based Synthesis for Time Window Temporal Logic Specifications via Mixed-Integer Linear Programming
}

\author{
Philip Smith$^{1,\ast}$, Ahmad Ahmad$^{1,\ast}$, Kevin Leahy$^1$
\thanks{$^\ast$ Equal contribution, $^1$ Department of Robotics Engineering, Worcester Polytechnic Institute, {\tt\small \{psmith3,aahmad4,kleahy\}@wpi.edu}}
\thanks{DISTRIBUTION STATEMENT A. Approved for public release; distribution is unlimited. OPSEC\# (Pending, NOT approved for Release)}
}

\begin{document}

\maketitle

\thispagestyle{empty}
\pagestyle{empty}

\begin{abstract}
Time Window Temporal Logic (TWTL) is a rich specification language for
cyber-physical systems that can compactly express sequential tasks with
explicit timing constraints.
In this paper, we consider the problem of synthesizing control inputs for
discrete-time linear systems subject to TWTL task specifications.
Building on the quantitative semantics (robustness) recently introduced
for TWTL in \cite{ahmad2023TWTLrobustness}, we encode the robust satisfaction
of a TWTL formula as a set of Mixed-Integer Linear constraints and pose
synthesis as a Mixed Integer Linear Program (MILP) that maximizes the robustness degree.
We prove that any feasible solution with positive objective value guarantees
Boolean satisfaction of the specification.
We address two synthesis settings: an \emph{open-loop} formulation that
optimizes the full control sequence from the initial state, and a
\emph{closed-loop} receding-horizon Model Predictive Controller (MPC) formulation that re-solves the
MILP at each step using the current measured state.
A key feature of our MPC formulation is a \emph{task-adaptive horizon}
that exploits the TWTL Deterministic Finite Automaton (DFA) to determine
the active sub-task at each step, limiting the prediction horizon to the
remaining window of the current task rather than the full formula horizon, this makes each re-solve
significantly cheaper than the initial open-loop solve.
\end{abstract}

\section{Introduction}\label{sec:intro}
Robot tasks often use spatio-temporal constraints: a surveillance drone which must dwell for a duration in multiple regions, or a multi-robot team which must complete sub-tasks in sequence while respecting timing windows. Temporal logics (TLs)~\cite{baier2008principles,konurSurveyTemporalLogics2013} address this by providing a language for stating what a system must accomplish over time. As such, they have become a standard tool for encoding high-level task specifications for cyber-physical systems.
Signal Temporal Logic (STL)~\cite{maler2004STLpaper} and Metric Temporal
Logic (MTL)~\cite{MTL1990} are widely used concrete-time logics.
Time Window Temporal Logic (TWTL)~\cite{Cristi2017TWTL} was introduced more
recently as an alternative that (i)~expresses sequential tasks compactly
through a dedicated concatenation operator, and (ii)~admits automata
translations whose complexity is \emph{independent} of the formula time
bounds, making it particularly attractive for automata-based planning and
synthesis~\cite{Cristi2020_TWTLrrt}.

Satisfying a TL specification requires synthesizing a control policy \cite{belta2019formal_review}. Optimization-based synthesis from TL specifications is enabled by a robustness measure quantifies Boolean satisfaction as a real value. Such a measure allows synthesis to optimize satisfaction margin, such as clearance of obstacles, alongside other optimization concerns like path length or battery life, and lends itself naturally to an MPC formulation. Optimizing robustness in this way allows plans to survive in a real, noisy world. 
For STL, the classical robustness~\cite{donze2010robustSTL} has been encoded
as Mixed-Integer Linear constraints by Raman et al.~\cite{raman2014model},
who also propose a receding-horizon MPC framework where the MILP is re-solved
at each step with a \emph{fixed} horizon equal to the formula time bound,
with the executed history pinned as equality constraints.
Until recently, TWTL lacked a robustness measure.
In~\cite{ahmad2023TWTLrobustness}, we introduced two sound quantitative semantics for TWTL--- a $\mathrm{min}$-$\mathrm{max}$-based robustness ($\rho$) and an Arithmetic Geometric Mean (AGM) robustness \cite{cristi2019AGMstl}---together with runtime monitors for partial runs.

A robustness measure only scores a given run; it does not produce one. What has been missing for TWTL is a synthesis procedure to optimize for $\rho$. In this paper, we leverage $\rho$ from~\cite{ahmad2023TWTLrobustness} to develop the first MILP synthesis procedure for TWTL.

\textbf{Contributions.}
\begin{inparaenum}
    \item We encode robustness as a recursive big-M
          MILP encoding and prove correctness
          (Sec.~\ref{sec:milp_encoding}, Theorem~\ref{thm:milp_correctness});
    \item We formulate an open-loop controller
          (Problem~\ref{prob:open_loop}) and a closed-loop MPC controller
          (Problem~\ref{prob:mpc}) with a \emph{task-adaptive horizon} that
          exploits the TWTL DFA to bound the prediction horizon rather than use the full formula horizon $T$ (Sec.~\ref{sec:mpc});
    \item We propose a warm-start strategy for the MPC re-solves
          that reduces the per-step online cost to linear in the number of TWTL sub-tasks for state updates
          and $O(1)$ for horizon shrinkage, plus the MILP re-solve on a
          progressively smaller problem
          (Sec.~\ref{subsec:warm_start}); and 
    \item We demonstrate the incresed efficiency of direct TWTL enconding by comparing against TWTL translated to STL (Sec.~\ref{subsec:open_loop_solve_time})
\end{inparaenum}

\textbf{Related work.}
While other works have previously used TWTL for monitoring \cite{b2022runtimeMonitForTWTL,bonnahQTWTLQualityAware2023}, or as a basis for sampling-based planning \cite{Cristi2017TWTL,Cristi2020_TWTLrrt,ahmad2026rrteta}, we perform synthesis directly using the specification and a model of system dynamics (Sec. \ref{subsec:dynamics}). Unlike sampling-based planners, which offer probabilistic completeness but no hard guarantees on robustness degree, our MILP formulation deterministically maximizes robustness and enables closed-loop MPC with disturbance rejection. The only prior work to use MILP with TWTL is~\cite{kamaleOptimalControlSynthesis2024}, which encodes automaton transitions (not robustness) as the MILP objective; ours is the
first robustness-maximizing MILP for TWTL.
Our encoding is directly analogous to~\cite{raman2014model} for STL, with
TWTL's hold and within operators playing the roles of STL's always and
eventually. This approach was also used successfully with another temporal logic (wSTL) in \cite{cardonaMixedIntegerLinear2023}. Our MPC horizon strategy differs from Raman et al.'s \cite{raman2014model} fixed-$H$ approach by
exploiting TWTL's automata structure to obtain a task-adaptive horizon,
a mechanism that is not available for STL without an automaton.

\section{Preliminaries}\label{sec:prelims}

\subsection{Dynamical System}\label{subsec:dynamics}

Consider a discrete-time linear time-invariant (LTI) system:
\begin{equation}\label{eq:lti_system}
    \begin{aligned}
        x_{t+1} &= A x_t + B u_t,\\
        o_t     &= C x_t,
    \end{aligned}
\end{equation}
where $x_t \in \mathcal{X} \subseteq \mathbb{R}^n$ is the state,
$u_t \in \mathcal{U} \subseteq \mathbb{R}^m$ is the control input,
and $o_t \in \mathcal{O} \subseteq \mathbb{R}^p$ is the observable output; $\mathcal{X}$ and
$\mathcal{U}$ are convex polytopes.
$\mathbf{o}_{t_1,t_2} := o_{t_1}o_{t_1+1}\cdots o_{t_2}$ denotes the
output word over $[t_1,t_2]$. We fix a finite set $\Pi = \{\pi_{1}, \dots, \pi_{k}\}$ of atomic propositions. Each $\pi_A \in \Pi$ is identified with a satisfaction region $A \subseteq \mathcal{O}$. An observation $o_t$ satisfies $\pi_A$ iff $o_t \in A$. In \cite{ahmad2023TWTLrobustness} a general predicate function $h$ is used to represent these atomic propositions, which we spoecialize to an affine form via Assumtion \ref{assum:affine_predicates}

\begin{assumption}\label{assum:affine_predicates}
All atomic propositions $\pi_A \in \Pi$ have affine predicate functions
$h(o) = \mathbf{c}_A^T o - b_A$, where $\mathbf{c}_A \in \mathbb{R}^p$
and $b_A, \sigma \in \mathbb{R}$, so that $A = \{o \mid h(o) > \sigma\}$.
Without loss of generality we absorb $\sigma$ into $b_A$, so that
$h(o_t) = \mathbf{c}_A^T C x_t - b_A$ throughout.
\end{assumption}

\subsection{Time Window Temporal Logic}\label{subsec:TWTL}

TWTL syntax is defined, inductively, as follows\cite{Cristi2017TWTL}.
\begin{equation}\label{eq:TWTL_syntax}
    \phi\;::=\;H^{d}s\mid H^{d}\neg s\mid\phi_1\wedge\phi_2
    \mid\phi_1\vee\phi_2\mid\neg\phi\mid\phi_1\cdot\phi_2
    \mid[\phi]^{[a,b]}
\end{equation}
where $s$ is an atomic proposition with affine predicate function by Assumption \ref{assum:affine_predicates}, $H^d$ is the \emph{hold} operator, $\cdot$ is the \emph{concatenation}
operator, and $[\phi]^{[a,b]}$ is the \emph{within} operator, with
$d,a,b \in \mathbb{Z}_{\geq 0}$ and $b \geq a$. From a given TWTL formula, the time horizion $T = ||\phi|| \in \mathbb{N}$ is computed, recursively, as follows \cite{ahmad2023TWTLrobustness}. 

\begin{equation}
    ||\phi||:=\begin{cases}
        \max(||\phi_1||,||\phi_2||); & \text{if }\phi\in\{\phi_1\wedge\phi_2,\phi_1\vee\phi_2\}\\
        ||\phi_1||;& \text{if } \phi = \neg \phi_1\\
        ||\phi_1||+||\phi_2||+\Delta t;&\text{if } \phi = \phi_1\cdot\phi_2\\
        d\Delta t;& \text{if }\phi=H^d\pi_{A}\\
        b;&\text{if }\phi = [\phi_1]^{[a,b]}
    \end{cases}
\end{equation}

The hold operator $H_d s$ specifies that $s \in AP$ should be repeated for $d$ time units. The semantics of $H_d \neg s$ is defined similarly, but for $d$ time units only symbols from $AP \setminus \{s\}$ should appear. The word $\mathbf{o}_{t_1, t_2}$ satisfies $\phi_1 \wedge \phi_2$, $\phi_1 \vee \phi_2$, or $\neg \phi$ if $\mathbf{o}_{t_1, t_2}$ satisfies both formulae, at least one formula, or does not satisfy the formula, respectively. The within operator $[\phi]_{[a,b]}$ bounds the satisfaction of $\phi$ to the time window $[a,b]$. The concatenation operator $\phi_1 \cdot \phi_2$ specifies that first $\phi_1$ must be satisfied, and then immediately $\phi_2$ must be satisfied \cite{Cristi2017TWTL}.

\subsection{TWTL DFA}\label{subsec:dfa}

Every TWTL formula $\phi$ translates to a DFA \cite{Cristi2017TWTL}
$\mathcal{A}_\phi = (Q, \Sigma, \delta, q_0, F)$, where $Q$ is the state
set, $\Sigma = 2^\Pi$, $\delta: Q\times\Sigma\to Q$, $q_0$ is the initial
state, and $F\subseteq Q$ is the accepting set.
A critical property of TWTL is that $|Q|$ is \emph{independent} of the
magnitude of the time bounds $a$, $b$, $d$~\cite{Cristi2017TWTL}.
The DFA state $q_{t+1} = \delta(q_t, l(o_t))$ is updated at each step
via a single table lookup ($O(1)$), where $l: \mathbb{R}^p \to 2^\Pi$.

\subsection{TWTL Robustness}\label{subsec:twtl_rho_recall}

\begin{definition}[TWTL Robustness~\cite{ahmad2023TWTLrobustness}]
\label{def:traditional_robustness}
Given $\phi$ and $\mathbf{o}_{t_1,t_2}$, the robustness
$\rho(\mathbf{o}_{t_1,t_2},\phi)$ is defined recursively as:
\begin{flalign}
\begin{aligned}\label{eq:TWTL_robustness}
    \rho(\mathbf{o}_{t_1,t_2},H^{d}\pi_{A}) &:=\begin{cases}
        \min\limits_{t\in[t_1,t_1+d]}h(o_t); & (t_2-t_1\geq d) \\
        \rho_{\bot}; & \text{otherwise}
    \end{cases} & \\
    \rho(\mathbf{o}_{t_1,t_2},\phi_1\wedge\phi_2)&:=
        \min\{\rho(\mathbf{o}_{t_1,t_2},\phi_1),\rho(\mathbf{o}_{t_1,t_2},\phi_2)\}&\\
    \rho(\mathbf{o}_{t_1,t_2},\phi_1\vee\phi_2)&:=
        \max\{\rho(\mathbf{o}_{t_1,t_2},\phi_1),\rho(\mathbf{o}_{t_1,t_2},\phi_2)\}&\\
    \rho(\mathbf{o}_{t_1,t_2},\neg\phi)&:=-\rho(\mathbf{o}_{t_1,t_2},\phi)&\\
    \rho(\mathbf{o}_{t_1,t_2},\phi_1\cdot\phi_2)&:= \\
        &\max_{t\in[t_1,t_2)}\!\left\{\min\bigl\{\rho(\mathbf{o}_{t_1,t},\phi_1),
        \rho(\mathbf{o}_{t+1,t_2},\phi_2)\bigr\}\right\}&\\
    \rho(\mathbf{o}_{t_1,t_2},[\phi]^{[a,b]})&:=\begin{cases}
        \max\limits_{t\geq t_1+a}\{\rho(\mathbf{o}_{t,t_1+b},\phi)\};
        & (t_2-t_1\geq b)\\
        \rho_{\bot}; & \text{otherwise}
    \end{cases} &
\end{aligned}
\end{flalign}
where $\rho_{\bot}$ is a large negative constant.
\end{definition}

\begin{lemma}[Soundness~\cite{ahmad2023TWTLrobustness}]\label{lem:soundness}
$\rho(\mathbf{o}_{t_1,t_2},\phi)>0 \implies \mathbf{o}_{t_1,t_2}\models\phi$
and
$\rho(\mathbf{o}_{t_1,t_2},\phi)<0 \implies
\mathbf{o}_{t_1,t_2}\not\models\phi$.
In words, a positive robustness means the word satisfies the specification.
\end{lemma}

\section{Problem Formulation}\label{sec:problem}
We introduce two synthesis problems: an open-loop problem and an MPC closed-loop problem. 
\begin{problem}[Open-Loop Synthesis]\label{prob:open_loop}
Given system~(\ref{eq:lti_system}) with initial state $x_0$, formula $\phi$
under Assumption~\ref{assum:affine_predicates}, and sets $\mathcal{X}$,
$\mathcal{U}$, find the control sequence
$\mathbf{u}^\star = u_0^\star \cdots u_{T-1}^\star$ solving:
\begin{equation}\label{eq:synthesis_opt}
\begin{aligned}
    \mathbf{u}^\star &= \argmax_{\mathbf{u}} \; \rho(\mathbf{o}_{0,T},\phi) \\
    \text{s.t.}\quad & \text{dynamics constrants from \eqref{eq:lti_system}} \\
\end{aligned}
\end{equation}
\end{problem}

\begin{remark}[Open-loop nature]\label{rem:open_loop}
Problem~\ref{prob:open_loop} computes $\mathbf{u}^\star$ once at $t=0$
from the fixed initial state $x_0$. By Lemma~\ref{lem:soundness}, any solution with
$\rho(\mathbf{o}_{0,T},\phi)>0$ guarantees $\mathbf{o}_{0,T}\models\phi$.
However, there is no mechanism to react to disturbances during execution. Therefore, we introduce problem \ref{prob:mpc} to address this limitation.

\end{remark}

\begin{problem}[Receding-Horizon MPC Synthesis]\label{prob:mpc}
Given system~(\ref{eq:lti_system}), formula $\phi$, and sets $\mathcal{X}$,
$\mathcal{U}$: at each time step $t \in [0,T{-}1]$, given the measured
state $x_t$, the DFA state $q_t$, the active task index $i(q_t)$, and
its task-adaptive horizon $H_t$ (Definition~\ref{def:adaptive_horizon}),
find a closed-loop receding horizon control that exploits the residual formula at the current DFA state.
\end{problem}

\section{MILP Encoding of TWTL Robustness}\label{sec:milp_encoding}

We show that Problems~\ref{prob:open_loop} and~\ref{prob:mpc} reduce to
MILPs.
Following~\cite{raman2014model}, for each sub-formula $\psi$ of $\phi$
evaluated over $\mathbf{o}_{t_1,t_2}$, we introduce a continuous variable
$r(\psi,t_1,t_2)$ representing $\rho(\mathbf{o}_{t_1,t_2},\psi)$.
The top-level variable $r(\phi,0,T)$ (or $r(\phi_t,t,t{+}H_t)$ in MPC)
is the MILP objective. A strong indicator of MILP performance is binary variable count, so we will mention the number of binary variables required for each operator (Remark \ref{rem:complexity}).


\subsection{Big-$M$ Constant}\label{subsec:bigM}

All encodings use $M > 0$ satisfying $M \geq \rho_{\top} - \rho_{\bot}$,
where $\rho_{\top}$, $\rho_{\bot}$ bound the achievable robustness over
$\mathcal{X}$.
Since $h(o_t) = \mathbf{c}_A^T C x_t - b_A$ is linear and $\mathcal{X}$
is a polytope, tight values are obtained by solving simple LPs over
$\mathcal{X}$ \emph{once offline}.

\subsection{Operator Encodings}\label{subsec:encodings}

\subsubsection{Predicate}
For $\psi = \pi_A$, introduce $r(\pi_A, t)\in\mathbb{R}$ for each $t$:
\begin{equation}\label{eq:enc_pred}
    r(\pi_A, t) = h(o_t) = \mathbf{c}_A^T C x_t - b_A.
\end{equation}
One linear equality per time step; \emph{no binary variables}. The sign of $r(\pi_A,t)$ encodes satisfaction.

\subsubsection{Hold (Temporal Conjunction)}
Analogue of STL's always $\square_{[0,d]}\pi_A$.
For $\psi = H^d\pi_A$ with $t_2-t_1\geq d$:
\begin{equation}\label{eq:enc_hold}
    r(H^d\pi_A,t_1,t_2) \leq r(\pi_A,\,t_1+k),
    \qquad k=0,1,\ldots,d.
\end{equation}
Tight at $\min_{k} r(\pi_A,t_1{+}k)$ under maximization.
\emph{No binary variables.}
If $t_2-t_1<d$, set $r(H^d\pi_A,t_1,t_2)=\rho_{\bot}$.

\subsubsection{Negation}
\begin{equation}\label{eq:enc_neg}
    r(\neg\varphi,t_1,t_2) = -\, r(\varphi,t_1,t_2).
    \quad\text{\emph{No binary variables.}}
\end{equation}

\subsubsection{Conjunction}
\begin{equation}\label{eq:enc_conj}
    r(\phi_1\wedge\phi_2,t_1,t_2) \leq r(\phi_i,t_1,t_2),\quad i=1,2.
\end{equation}
Tight at minimum under maximization.
\emph{No binary variables.}

\subsubsection{Disjunction}
Using binary $z^{\vee}\in\{0,1\}$:
\begin{flalign}
\begin{aligned}\label{eq:enc_disj}
    &r(\phi_1\vee\phi_2,t_1,t_2) \geq r(\phi_i,t_1,t_2),\quad i=1,2,\\
    &r(\phi_1\vee\phi_2,t_1,t_2) \leq r(\phi_1,t_1,t_2)
        + M(1-z^{\vee}),\\
    &r(\phi_1\vee\phi_2,t_1,t_2) \leq r(\phi_2,t_1,t_2)
        + M\,z^{\vee}.
\end{aligned}
\end{flalign}
\emph{One binary variable} per disjunction.

\subsubsection{Within (Temporal Disjunction)}
Analogue of STL's eventually $\lozenge_{[a,b]}\varphi$.
Using binary $z^W_t\in\{0,1\}$ for each $t\in[t_1{+}a,\,t_1{+}b]$:
\begin{align}
\begin{aligned}\label{eq:enc_within}
    \sum_{t=t_1+a}^{t_1+b} z^W_t &= 1,\\
    r([\varphi]^{[a,b]},t_1,t_2) &\geq r(\varphi,t,t_1+b), \quad\forall t\in[t_1{+}a,\,t_1{+}b],\\
    r([\varphi]^{[a,b]},t_1,t_2) &\leq r(\varphi,t,t_1+b) + M(1-z^W_t), \\
        &\hphantom{\geq r(\varphi,t,t_1+b)}\forall t\in[t_1{+}a,\,t_1{+}b].
\end{aligned}
\end{align}
\emph{$b{-}a{+}1$ binary variables.}
If $t_2-t_1<b$, set $r([\varphi]^{[a,b]},t_1,t_2)=\rho_{\bot}$.

\subsubsection{Concatenation}
Using binary $z^C_t,z^S_t\in\{0,1\}$ for each split $t\in[t_1,t_2)$ and
auxiliary inner-minimum variables $r^C_t$:
\begin{flalign}
\begin{aligned}\label{eq:enc_cat}
    &z^C_t \geq z^C_{t-1} + z^S_t,\quad\forall t\in[t_1,t_2), \\
    &z^C_t \leq z^C_{t-1} + M(z^S_t),\quad\forall t\in[t_1,t_2), \\
    &z^C_0 = 0,\qquad \sum_{t=t_1}^{t_2-1} z^S_t = 1 \\
    &r^C_t \leq r(\phi_1,t_1,t) + M(z^C_t),\quad\forall t\in[t_1,t_2),\\
    &r^C_t \geq r(\phi_1,t_1,t) - M(z^C_t),\quad\forall t\in[t_1,t_2),\\
    &r^C_t \leq r(\phi_2,t,t_2) + M(1-z^C_t),\quad\forall t\in[t_1,t_2),\\
    &r^C_t \geq r(\phi_2,t,t_2) - M(1-z^C_t),\quad\forall t\in[t_1,t_2),\\
    &r(\phi_1\cdot\phi_2,t_1,t_2) \leq r^C_t,\quad\forall t\in[t_1,t_2),\\
\end{aligned}
\end{flalign}
\emph{$2*(t_2-t_1)$ binary variables} per concatenation.

\subsection{Full MILP (Open-Loop)}\label{subsec:full_milp}

\begin{flalign}
\begin{aligned}\label{eq:full_milp}
    \max\quad &r(\phi,0,T)\\
\text{s.t.}\quad & \text{dynamics constrants from \eqref{eq:lti_system}} \\    &\text{constraints }(\ref{eq:enc_pred})\text{--}(\ref{eq:enc_cat})
     \;\forall\text{ sub-formulae of }\phi.
\end{aligned}
\end{flalign}

\begin{theorem}[MILP Correctness]\label{thm:milp_correctness}
Let $(\mathbf{u}^\star,\mathbf{x}^\star,\mathbf{o}^\star)$ be an optimal
solution to~(\ref{eq:full_milp}) with $r(\phi,0,T) > 0$.
Then $\rho(\mathbf{o}^\star_{0,T},\phi) > 0$, and hence
$\mathbf{o}^\star_{0,T}\models\phi$.
\end{theorem}
\begin{proof}
Structural induction on $\phi$.\\
\textit{Predicate}: $r(\pi_A,t)=h(o_t)=\rho(\mathbf{o}_t,\pi_A)$ by~(\ref{eq:enc_pred}). \textit{Hold}: constraints~(\ref{eq:enc_hold}) give
$r(H^d\pi_A,\cdot)\leq r(\pi_A,t_1{+}k)$ for all $k$; tight at
$\min_k h(o_{t_1+k})=\rho(\mathbf{o}_{t_1,t_2},H^d\pi_A)$. \textit{Negation}: direct from~(\ref{eq:enc_neg}). \textit{Conjunction}: upper bounds in~(\ref{eq:enc_conj}) tighten to the
minimum. \textit{Disjunction}: lower bounds force $r\geq\max\{r(\phi_1),r(\phi_2)\}$;
big-$M$ upper bound with $z^\vee$ enforces equality.\\
\textit{Within}: sum constraint selects $t^\star$; lower bound and big-$M$
upper bound enforce $r=r(\varphi,t^\star,t_1{+}b)$
\textit{Concatenation}: upper bounds on $r^C_t$ recover the inner minimum;
sum and big-$M$ constraints enforce $r=r^C_{t^\star}$ at the maximizing
split.
Soundness follows from Lemma~\ref{lem:soundness}.
\end{proof}

\begin{remark}[Binary variable count]\label{rem:complexity}
For $\phi = [\phi_1]^{[a_1,b_1]}\cdot\cdots\cdot[\phi_k]^{[a_k,b_k]}$,
the open-loop MILP has $O(kT)$ binary variables, dominated by the
concatenation operators.
Incorporating the DFA automaton structure to replace split-point binaries
with automaton transition constraints reduces this to $O(k\log T)$
following~\cite{kurtzMixedIntegerProgrammingSignal2022}; we leave this as future work.
\end{remark}

\section{Receding-Horizon MPC}\label{sec:mpc}

We now develop the closed-loop synthesis of Problem~\ref{prob:mpc},
introducing the task-adaptive horizon, the residual formula, and the
warm-start re-solve strategy that together make the MPC tractable (Algorithm \ref{alg:mpc}).

\subsection{Residual Formula}\label{subsec:residual}

\begin{definition}[Residual Formula]\label{def:residual}
Given formula $\phi$ with DFA $\mathcal{A}_\phi$ and current DFA state
$q_t \in Q$, the \emph{residual formula} $\phi_t$ is the sub-formula of
$\phi$ that remains to be satisfied from $q_t$.
It is determined by the accepting paths of $\mathcal{A}_\phi$ from $q_t$:
\begin{equation}
    \mathbf{o}_{t,T}\models\phi_t \;\iff\;
    \delta^\ast(q_t, \mathbf{o}_{t,T}) \in F,
\end{equation}
where $\delta^\ast$ is the DFA transition function.
\end{definition}

For the typical TWTL formula
$\phi = [\phi_1]^{[a_1,b_1]}\cdot[\phi_2]^{[a_2,b_2]}\cdots
[\phi_k]^{[a_k,b_k]}$, the DFA state $q_t$ directly identifies the active
task index $i(q_t) \in \{1,\ldots,k\}$, so:
\begin{equation}
    \phi_t = [\phi_i]^{[a_i',b_i']}\cdot[\phi_{i+1}]^{[a_{i+1},b_{i+1}]}
    \cdots[\phi_k]^{[a_k,b_k]},
\end{equation}
where $[a_i', b_i'] = [a_i - (t - t_i^s),\; b_i - (t - t_i^s)]$ are
the time-shifted bounds for task $i$, and $t_i^s$ is the time task $i$ started.

\begin{remark}[Cost of residual formula extraction]\label{rem:residual_cost}
Given $q_t$, extracting $\phi_t$ costs $O(1)$ (a DFA table lookup to
identify $i(q_t)$) plus $O(k)$ to construct the remaining task chain and
adjust the time bounds $[a_i', b_i']$ for task $i$.
This is negligible relative to the MILP re-solve.
\end{remark}

\begin{remark}[Alternate residual formula] \label{rem:alt_residual}
An alternate approach for extracting the residual formula is to adapt the strategy developed in \cite{b2022runtimeMonitForTWTL} with periodic rewriting of the formula as part of the monitoring efforts. Whereas \cite{b2022runtimeMonitForTWTL} employs rewriting as an online monitoring effort, our proposed adaptation utilizes rewriting as a simplification method in order to reduce the computational load of a replan. It would do this by removing parts of the formula that have already been satisfied, reducing the number of unnecessary binary variables. This would help with both local re-plans within a DFA state, as it reduces the number of time points still in consideration, as well as a global re-plan due to the removal of already completed tasks. We leave implemenation and benchmarking the effectiveness of this approach as future work. 
\end{remark}
\subsection{Task-Adaptive Horizon}\label{subsec:adaptive_horizon}

The full formula horizon $T = \|\phi\|$ is typically large (the sum of all
task windows).
Solving the MILP over $[t, T]$ at every step is expensive and unnecessary,
since the structure of $\phi$ decomposes into sequential tasks 

\begin{definition}[Task-Adaptive Horizon]\label{def:adaptive_horizon}
Given active task $i = i(q_t)$ started at time $t_i^s$, and global safety
constraints with horizon $T_{\mathrm{safe}}$, the \emph{task-adaptive
horizon} at time $t$ is: $H_t \;:=\; \min\bigl(b_i - (t - t_i^s),\; T - t\bigr)$,
the smaller of the remaining window of task $i$ and the total remaining
horizon.
\end{definition}
\begin{remark}[Comparison with Raman et al.~\cite{raman2014model}]
\label{rem:raman_comparison}
Raman et al.\ use a \emph{fixed} horizon $H = T$ at every step, with
the executed history $x_0,\ldots,x_{t-1}$ pinned as equality constraints,
and re-encode the \emph{full} formula $\phi$ at every re-solve.
Problem~\ref{prob:mpc} differs in two ways.
First, it uses the \emph{task-adaptive} horizon $H_t \leq \max_i b_i$
(Definition~\ref{def:adaptive_horizon}), which is in general much smaller
than $T$ and resets at each task transition rather than shrinking monotonically.
Second, it encodes only the \emph{residual} formula $\phi_t$
(Definition~\ref{def:residual}), which is structurally simpler than $\phi$
once sub-tasks are completed.
Both differences are enabled by TWTL's automata structure and are not
available for general STL without an automaton.
\end{remark}

\begin{remark}[Why $H_t \leq \max_i b_i$, not $T$]\label{rem:horizon_bound}
The key observation is that task $i$ must be completed within
$b_i - (t - t_i^s)$ steps, regardless of what follows.
Planning beyond this window does not help task $i$ and unnecessarily
inflates the MILP.
Once task $i$ completes and the DFA advances, the horizon \emph{resets}
to $b_{i+1}$ for task $i+1$ rather than continuing to shrink from $T$.
The maximum prediction horizon at any step is therefore
$\max_i b_i$, \emph{not} $T$.
\end{remark}

\begin{remark}[Three horizon strategies]\label{rem:horizon_comparison}
Table~\ref{tab:horizon} compares the three strategies for the example
formula $\phi = ([H^4\pi_A]^{[0,8]}\cdot[H^4\pi_B]^{[0,10]}\cdot
[H^3\pi_C]^{[0,11]})\wedge H^{50}\neg\pi_O$
($T \approx 50$, $\max_i b_i = 11$).
\end{remark}

\begin{table}[t]\scriptsize
\centering
\caption{Horizon strategy comparison for the running example
($T=50$, $\max_i b_i=11$, $k=3$ tasks).
Binary variable count excludes the safety constraint (zero cost).}
\label{tab:horizon}
\renewcommand{\arraystretch}{1.25}
\begin{tabular}{@{}lcc@{}}
\toprule
\textbf{Strategy} & \textbf{Horizon at step $t$} &
\textbf{Binary vars (approx.)} \\
\midrule
Raman et al.\ (fixed $H=T$) & Always 50 & $\sim 150$, constant \\
Shrinking                    & $T - t$, starts at 50 & Decreasing from 150 \\
\textbf{Task-adaptive (ours)} & $\leq 11$, resets per task &
$\leq 30$, always small \\
\bottomrule
\end{tabular}
\end{table}

\subsection{Warm-Start Strategy}\label{subsec:warm_start}

At each step, the MILP must be rebuilt with the updated initial state $x_t$
and the shrunken time window $[a_i', b_i']$.
Rebuilding from scratch is expensive; we instead use a warm-start strategy
that makes the per-step overhead negligible.

\textbf{Offline precomputation.}
The MILP constraint \emph{matrices} depend on $\phi_t$ and $H_t$, not on
$x_t$.
Because the DFA has finitely many states ($|Q|$ independent of time bounds),
we precompute one parametric MILP per DFA state $q \in Q$ using the full
task window $b_i$ as the horizon.
This produces a set of matrices $\{(A_q, G_q, h_q)\}_{q \in Q}$ encoding
the equality, inequality, and integrality constraints for each residual
formula, computed \emph{once offline before execution}.

\textbf{Online update.}
At each step $t$, the online cost has three components:

\begin{enumerate}
    \item \emph{State update ($O(n)$).}
          Update the initial-condition equality:
          $x_{t+1} = Ax_t + Bu_t$ is the only RHS entry that changes.
          This is a vector update of size $n$.

    \item \emph{Window shrinkage ($O(1)$).}
          As the task window shrinks by one step, the earliest valid start
          time for the within operator increases by one.
          Fix the corresponding binary variable $z^W_{t_1+a'} = 0$
          (the newly expired start time) as a simple bound update, rather
          than removing the constraint.
          Cost: update one variable bound.

    \item \emph{Task transition ($O(n)$).}
          When the DFA advances from $q_t$ to a new state $q_{t+1}$
          (task $i$ completes), swap the precomputed MILP matrices from
          $(A_{q_t}, G_{q_t}, h_{q_t})$ to $(A_{q_{t+1}}, G_{q_{t+1}},
          h_{q_{t+1}})$ and update the initial state.
          This is an $O(1)$ pointer swap plus an $O(n)$ RHS update.
\end{enumerate}

After the RHS and bound updates, the MILP is warm-started from the previous
step's optimal solution, which is feasible for the updated problem up to
minor perturbations introduced by $x_t$.
Modern solvers (Gurobi \cite{gurobi}, MOSEK\cite{mosek}) exploit warm starts aggressively, typically
resolving in a fraction of the cold-start time when the problem changes only
slightly.

\begin{remark}[Total online cost per step]
\label{rem:online_cost}
The online overhead per MPC step --- excluding the MILP re-solve --- is
$O(n)$. This is identical to the parametric MPC observation in ~\cite{raman2014model} but applied to a \emph{progressively smaller}
MILP; binary variable count decreases as tasks complete, therefore, re-solve times
tend to shorten over the execution horizon.
\end{remark}

\subsection{MPC Algorithm}\label{subsec:mpc_algo}

\begin{algorithm}[t]\scriptsize
\caption{Task-Adaptive Receding-Horizon MPC for TWTL Synthesis}
\label{alg:mpc}
\DontPrintSemicolon
\KwIn{System~(eq. \ref{eq:lti_system}), $x_0$, formula $\phi$,
      DFA $\mathcal{A}_\phi$, precomputed MILP matrices
      $\{(A_q,G_q,h_q)\}_{q\in Q}$}
\KwOut{Applied control sequence $u_0, u_1, \ldots$}
\BlankLine
\textbf{Offline:} Precompute MILP matrices $\{(A_q,G_q,h_q)\}_{q\in Q}$;
compute $\rho_\top$, $\rho_\bot$ via LP\;
\BlankLine
Initialize DFA: $q_0 \leftarrow q_{\mathrm{init}}$;
$t_i^s \leftarrow 0$; warm-start solution $\mathbf{u}^{\mathrm{ws}}
\leftarrow \mathbf{0}$\;
\For{$t = 0, 1, \ldots, T{-}1$}{
    $i \leftarrow i(q_t)$\tcp*{Active task index: $O(1)$}
    $H_t \leftarrow \min(b_i - (t - t_i^s),\; T-t)$
    Load MILP for $q_t$: matrices $(A_{q_t}, G_{q_t}, h_{q_t})$\;
    Update RHS with $x_t$\tcp*{$O(n)$}
    Fix $z^W_{t'} = 0$ for expired start times $t' < t + a_i'$
    Warm-start solver from $\mathbf{u}^{\mathrm{ws}}$\;
    Solve MILP over $[t,\, t{+}H_t]$\;
    \If{infeasible \textbf{or} $r(\phi_t, t, t{+}H_t) \leq 0$}{
        \textbf{Report} infeasibility at step $t$; \textbf{break}\;
    }
    Apply $u_t \leftarrow u_t^\star$; observe $x_{t+1}$, $o_t = Cx_t$\;
    $\mathbf{u}^{\mathrm{ws}} \leftarrow $ shift optimal sequence by 1\;
    $q_{t+1} \leftarrow \delta(q_t,\, l(o_t))$\tcp*{DFA update: $O(1)$}
    \If{$q_{t+1} \in F$}{
        \textbf{Report} $\phi$ satisfied; \textbf{break}\;
    }
    \If{$i(q_{t+1}) \neq i(q_t)$}{
        $t_{i+1}^s \leftarrow t+1$
    }
}
\end{algorithm}

\begin{theorem}[MPC Satisfaction Guarantee]\label{thm:mpc_correctness}
If Algorithm~\ref{alg:mpc} does not report infeasibility and
$r(\phi_t, t, t{+}H_t) > 0$ at every step $t$, then
$\mathbf{o}_{0,T}\models\phi$.
\end{theorem}
\begin{proof}
At each step $t$, by Theorem~\ref{thm:milp_correctness},
$r(\phi_t,t,t{+}H_t)>0$ implies
$\rho(\mathbf{o}_{t,t+H_t},\phi_t)>0$, hence
$\mathbf{o}_{t,t+H_t}\models\phi_t$ by Lemma~\ref{lem:soundness}.
By Definition~\ref{def:residual}, satisfaction of $\phi_t$ from $q_t$
means $\delta^\ast(q_t,\mathbf{o}_{t,t+H_t})\in F$.
The DFA correctly tracks execution via
$q_{t+1}=\delta(q_t,l(o_t))$, so the full trajectory $\mathbf{o}_{0,T}$
reaches an accepting state, i.e., $\mathbf{o}_{0,T}\models\phi$.
\end{proof}

\begin{remark}[Connection to runtime monitors~\cite{ahmad2023TWTLrobustness}]
\label{rem:monitor_trigger}
The runtime robustness monitors of~\cite{ahmad2023TWTLrobustness} provide
an interval $[\rho](\mathbf{o}_{0,t},\phi)$ at each step.
When the lower bound $\underline{\rho}$ of this interval approaches zero,
it signals the current trajectory is at risk.
This provides a principled, monitor-triggered replanning condition:
re-solve only when $\underline{\rho} \leq \epsilon$ for a threshold
$\epsilon > 0$, reducing unnecessary MILP solves in disturbance-free
segments.
\end{remark}

\section{Numerical Example}\label{sec:example}
\subsection{Experimental Setup}\label{subsec:setup}
For ease of comparison, we will use the same TWTL formula for all of the following examples:
\begin{equation}\label{eq:TWTL_specification}
    \phi = ([H^4\pi_A]^{[0,8]}\cdot[H^4\pi_B]^{[0,10]}\cdot[H^3\pi_C]^{[0,11]} )\land H^{50} \lnot \pi_O
\end{equation}

All of the following examples take place in a $20\times20$ continuous $\mathbb{R}^2$ workspace, where the agent moves under double integrator dynamics discretized at $\Delta t = 1$. The propositions $\pi_A$, $\pi_B$, and $\pi_C$ correspond to three square regions of interest the agent must visit and hold in, with a single square obstacle placed in the workspace. The agent starts at the origin $(0,0)$ for every run.

Each specification is encoded as a MILP and solved with Gurobi\cite{gurobi}. All timing results were collected on an Apple M1 Pro MacBook Pro with 16 GB of RAM, and reported solve times are wall clock times.

\subsection{Robustness Relevance}\label{subsec:robustness}
To show the importance of having a robustness measure, we compare planned paths: one where the solver is not optimizing for robustness, and one where it is. While both satisfy the specification, the robustness based path is qualitatively and quantitatively more robust to failures (Fig. \ref{fig:robustness_importance}).
\begin{figure}
    \centering
    \includegraphics[width=\linewidth]{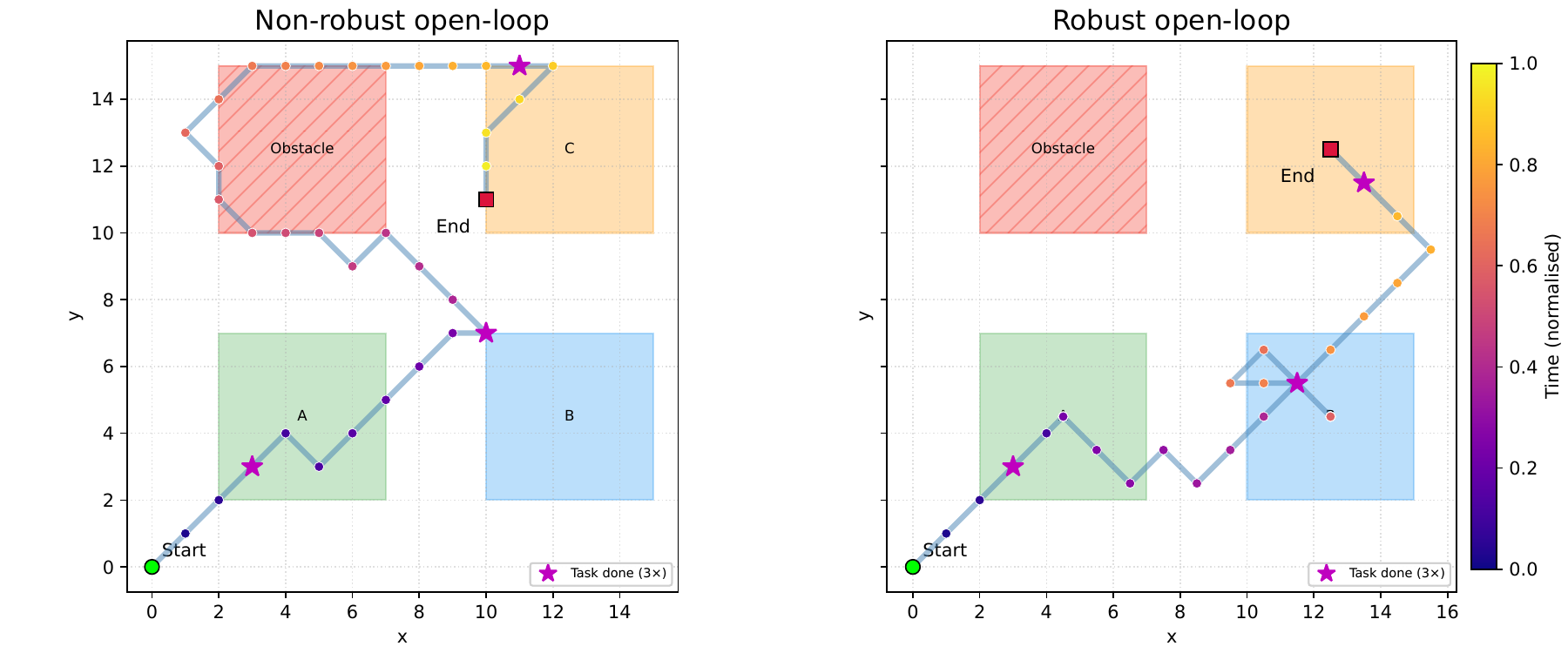}
    \caption{Open loop solving, with non-robustness based path (L) and robustness based path (R) using the specification from eq. \eqref{eq:TWTL_specification}}
    \label{fig:robustness_importance}
\end{figure}

\subsection{Open Loop Solve Time Comparison with STL}\label{subsec:open_loop_solve_time}
To compare solving times between TWTL and STL, we translate TWTL specifications into STL specifications.
In STL, \eqref{eq:TWTL_specification} translates to: $        \phi = \lozenge_{[0,4]}( \square_{[0,4]}(\pi_A) \land \lozenge_{[0,6]}(\square_{[0,4]}(\pi_B) \land \lozenge_{[0,8]}(\square_{[0,3]}(\pi_C)))) \land \square_{[0,50]}(\neg\pi_O) $ 

We benchmark TWTL and STL solve times for randomized number goals and task durations with semantically equivalent specifications (Fig.  \ref{fig:TWTLvsSTL_solvetime}). For smaller numbers of tasks, they are about equal with STL performing slightly better on average, but with four or more subtasks, TWTL performs better due to the reduction in binary variables.
\begin{figure}
    \centering
    \includegraphics[width=0.8\linewidth]{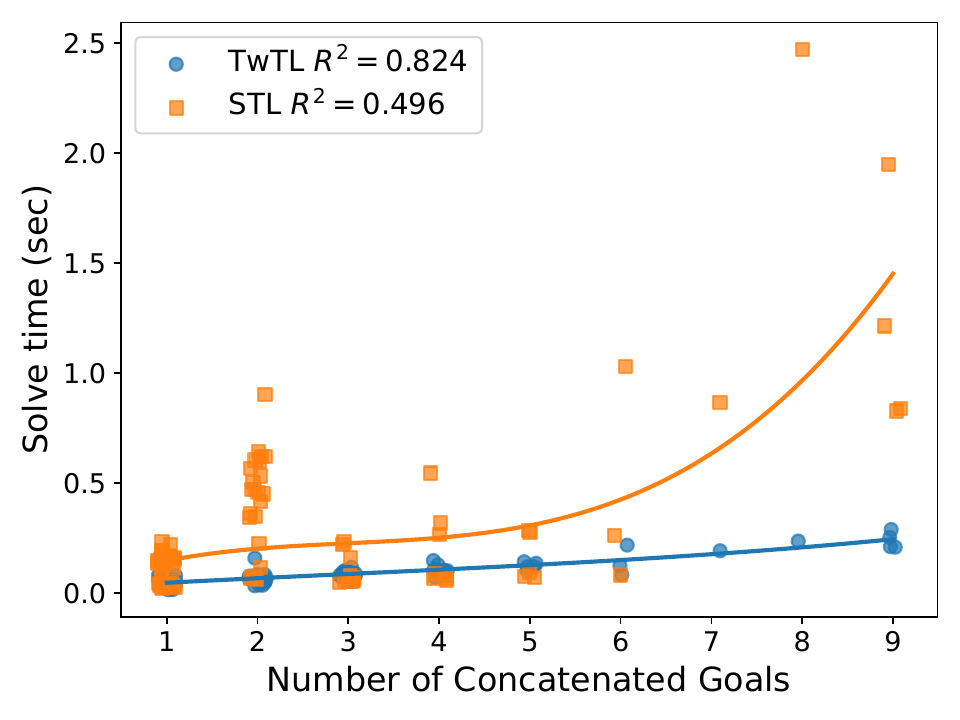}
    \caption{Solve time relative to number of concatenated goals. A 3rd order regression was applied for the plotted line, and then a 0.1 random jitter was applied to the $x$-axis for readability}
    \label{fig:TWTLvsSTL_solvetime}
\end{figure}

\subsection{Open Loop vs Closed Loop}\label{subsec:openloop}
To demonstrate the importance of closed-loop control, we introduce a disturbance to the path, displacing the agent down as it attempts to follow the open loop control inputs causing it to fail, while MPC is able to adapt and replan, completing all three tasks (Fig. \ref{fig:disturbance_handling}).
\begin{figure}
    \centering
    \includegraphics[width=0.9\linewidth]{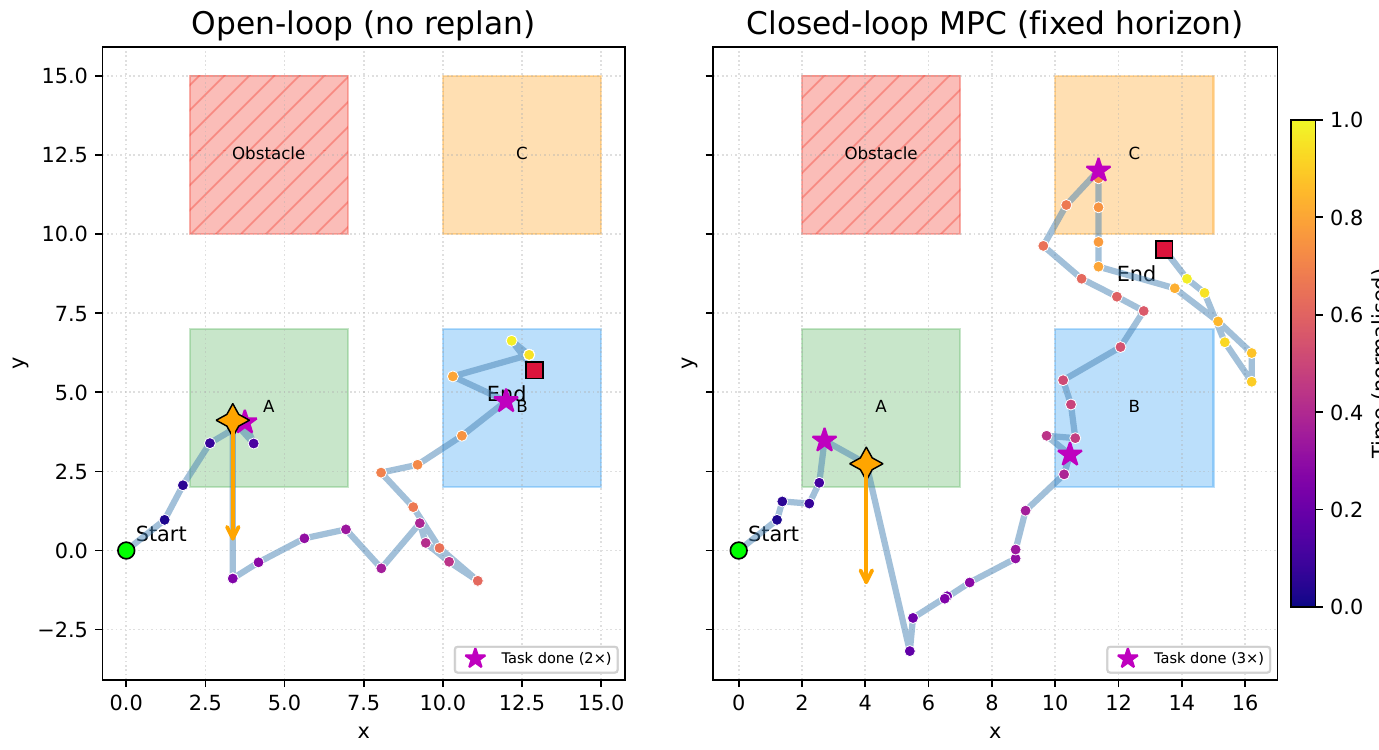}
    \caption{Comparing recovery of Open loop (L) vs Closed loop (R) MPC after a disturbance. Disturbance happened at time step 6, notated with an orange star, and was unknown prior to planning. This again follows the specification in eq. \ref{eq:TWTL_specification}. The open loop trajectory follows its original plan which just happens to wander into region B}
    \label{fig:disturbance_handling}
\end{figure}

\subsection{Fixed Horizion vs DFA approach}\label{subsec:fixed_horizion}


An advantage of TWTL is the ability to utilize a DFA for closed-loop planning and control (Sec. \ref{subsec:dfa}).
DFA replanning outperforms the fixed horizon approach, with increased benefit as the number of tasks increases (Fig. \ref{fig:replan_fixed_horizon_vs_DFA}).

\begin{figure}
    \centering
    \includegraphics[width=0.8\linewidth]{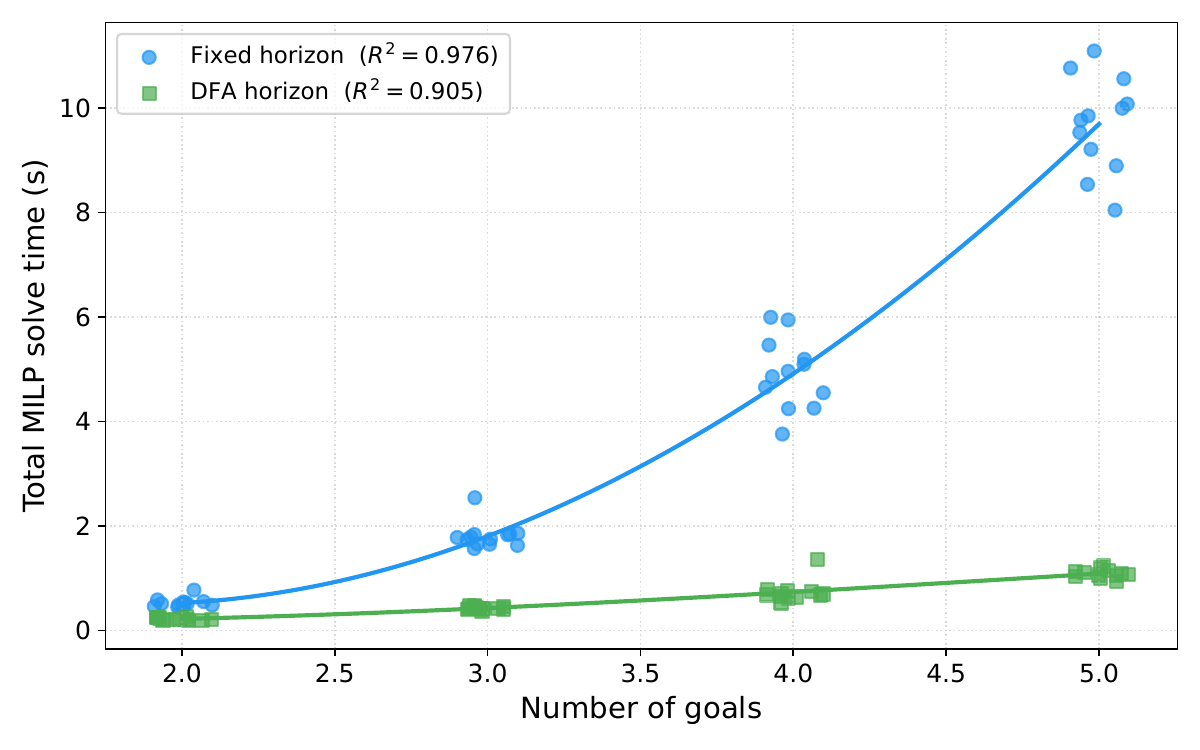}
    \caption{Total solve time over the course of a run comparing a fixed horizion approach (analogous to \cite{raman2014model}) with our proposed DFA approach over an increasing number of sequential goals.}
    \label{fig:replan_fixed_horizon_vs_DFA}
\end{figure}


\section{Conclusion and Future Work}\label{sec:conclusion}

We presented a robustness-maximizing MILP synthesis framework for TWTL
specifications, comprising an open-loop trajectory optimizer and a
closed-loop receding-horizon MPC controller.
The MILP encoding extends the approach of~\cite{raman2014model} for STL
to TWTL's hold, within, and concatenation operators, with a provable
correctness guarantee via soundness of $\rho$. We also showed that for problems which are best expressed as multiple sequential subtasks, MILP based synthesis of TWTL is more efficient than MILP based synthesis of STL.
The MPC formulation introduces a task-adaptive prediction horizon that
exploits the TWTL DFA to bound the per-step MILP size by
$\max_i b_i$ rather than $T$---a feature not available for STL---and a
warm-start strategy that reduces per-step online overhead to $O(n)$.

In future work, we plan to \begin{inparaenum}
    \item Exploit the DFA structure to further reduce
binary variables via logarithmic encoding~\cite{kurtzMixedIntegerProgrammingSignal2022}
(Remark~\ref{rem:complexity});
    \item Extend to the AGM robustness $\eta$
of~\cite{ahmad2023TWTLrobustness} via a mixed-integer second-order cone program (MISOCP) formulation;
    \item Utilize rewriting on both current state and whole specifications to reduce binary variables following. \cite{b2022runtimeMonitForTWTL}; and
    \item Apply the framework to multi-agent planning with TWTL task allocation.
\end{inparaenum}

\bibliographystyle{IEEEtran}
\bibliography{main_TWTLMILP}

\end{document}